\documentclass{article}
\usepackage{amsmath}
\usepackage{amsfonts}
\usepackage{amssymb}
\usepackage{mathrsfs}  
\usepackage{amsthm}
\newtheorem{theorem}{Theorem}
\newtheorem{theorema}{Theorem}

\usepackage{graphicx}
\usepackage{epigraph}

\usepackage{setspace}
\usepackage{moresize}

\usepackage{pgfplots}
\usepackage{tikz}
\usetikzlibrary{trees}
\tikzset{
  treenode/.style = {shape=rectangle, rounded corners,
                     draw, align=center,
                     top color=white, bottom color=blue!20},
  root/.style     = {treenode, font=\Large, bottom color=red!30},
  env/.style      = {treenode, font=\ttfamily\normalsize},
  dummy/.style    = {circle,draw}
}

\usepackage{hyperref}
\hypersetup{
  pdftitle={Against Proxy Optimization},
  hidelinks,
  }
 
\usepackage[backend=bibtex, style=authoryear, useprefix, maxcitenames=2, maxbibnames=10, giveninits=true]{biblatex}
\bibliography{optimize.bib} 
 
\title{\textsc{Against Proxy Optimization}\thanks{Forthcoming in \emph{Philosophy and Phenomenological Research}. Thanks to Mathias B{\"o}hm, Mikayla Kelley, Lucy Mason, Aydin Mohensi and David Thorstad for helpful comments on previous drafts. Thanks to Mel Andrews, Kelli Bar, Vincent Conitzer, Julien Dutant, Mario G{\"u}nther, James Hodgson, Wesley Holliday, Lucy James, Dejan Makovec, Siddharth Muthukrishnan, John Schindler, James Shaw, David Wallace, Cole Wyeth and audiences at the Center for Philosophy of Science at Pittsburgh, Agent Foundations Workshop at CMU, Formal and Experimental Philosophy Workshop at University of Buffalo, 99th Joint Session of the Aristotelian Society and the Mind Association at University of Glasgow, Workshop on Social Choice for AI at IASEAI'26 in Paris, my spring 2026 seminar on AI and `Rational for What?' at National University of Singapore for helpful discussion.}}
\author{Sven Neth \\ University of Pittsburgh}
\date{}

\begin{document}

\maketitle

\begin{abstract}
I discuss conditions under which maximizing a proxy utility function is harmful and suggest this poses problems for applying decision theory.
\end{abstract}

\section{Introduction}

Value is complex and hard to measure, so we often measure and try to improve a simplified proxy metric. For example, we care about education but measure test scores. We care about health but measure step count. We care about welfare but measure GDP. We care about good research but measure citation counts. In all of these cases, the proxy metric is related to the value that we care about but it's not the whole story. 

It's easy to feel like something is going wrong when we replace complex values by simplified metrics. \textcite{Nguyen2024} describes this phenomenon as \emph{value capture} and argues that it is problematic because the proxy metric distorts the values we care about. Economists discuss related phenomena as \emph{Goodhart's Law}, which says that when a measure becomes a target, it ceases to be a good measure \parencite{Goodhart1975, Kerr1975}. Researchers in machine learning worry about \emph{reward hacking}, which is when maximizing a simplified reward function leads to unexpected and often undesirable behavior \parencite{Hubinger2019, Skalse2022, Karwowski2024} and \emph{wireheading}, when agents manipulate their reward signal \parencite{Everitt2021}.

While I agree with Nguyen that something is going wrong in these cases, it is not entirely clear what the problem is. Nguyen is surely right that proxy metrics distort our values. By definition, they don't capture all aspects of our values. But there can also be advantages to using proxies. For example, we can use them to make decisions when true values are too complicated to measure. So it's not clear whether proxies are bad all things considered. The question I want to focus on is: under what conditions does maximizing proxies lead to bad outcomes? Following \textcite{John2024}, I propose \emph{proxy failure} as term for this phenomenon.
 
My goal is to make progress on the question when proxy failure happens by studying a model introduced by \textcite{Zhuang2020}. In this model, we have a true utility function and a proxy utility function. They claim that, given a plausible condition, maximizing the proxy eventually leads true utility to go down. This undercuts the pragmatic case for proxies: we can use them to make decisions but doing so will eventually be bad from the point of view of what we care about. I critically discuss the condition proposed by \textcite{Zhuang2020} and argue that two other conditions do a better job of capturing proxy failure. I end by some reflections on how to avoid proxy failure.

\section{Proxy Optimization}

Let us start by looking at the model of proxy optimization introduced by \textcite{Zhuang2020}. The key feature is a distinction between true utility and proxy utility. They want to model the problem of AI alignment: we have a true utility function but can't write it down so we tell an AI system to maximize some proxy.\footnote{This assumes that we can model AI systems as utility maximizers which is controversial \parencite{Bales2025}. I will sidestep this debate here.} However, their model is more general and can represent proxy optimization in other contexts as well, like when we care about education but measure test scores.

Assume that the state of the world can be represented as a vector of $n$ real-valued features. These features represent various aspects of the world relevant to how good things are. For example, one feature might represent the well-being of one person and another feature might represent the well-being of another person. Features can also represent global properties, for example how much happiness and beauty there are. Given the assumption that features are real-valued, the state of the world can be represented as point in $n$-dimensional Euclidean space. We write $s_i$ for the value of the $i$-th feature in state $s$. Each feature $i$ has a lower bound $b_{i}$ which represents the lowest possible value this feature can obtain. This model connects to recent work on value pluralism since we can think of features as  dimensions of value \parencite{Hedden2024}. However, features might also represent aspects of the world which are not themselves dimensions of value.

In this model, states are different from the way \textcite{Savage1972} thinks about states. For Savage, we are uncertain about the state but can't change it. Decision theory is about finding the best course of action given our uncertainty. In contrast, \textcite{Zhuang2020} model optimization as a process which changes the state to increase utility. In addition, there is no uncertainty in the model. There are ways to incorporate uncertainty: we could model uncertainty as part of the state or add a probability measure over states. I briefly sketch one way to do so in the appendix. However, for ease of exposition I will mostly bracket uncertainty.

Assume there is a \emph{true utility function} $u : \mathbb{R}^n \to \mathbb{R}$. This utility function is sensitive to all features and tells us how good the state of the world is. The true utility function can represent different kinds of value: moral, aesthetic, prudential and so on. For our purposes, it does not matter which of these interpretations we adopt. The key idea is that the utility function represents some complex value which can be sensitive to many features of the world. The true utility function encodes how to make tradeoffs between features. If features represent the well-being of different people, this assumes we have solved the problem of interpersonal utility comparison \parencite{Sen1999, Nebel2023}. There are analogous problems for aggregating dimensions of value more generally, which I will sidestep here. My focus is on an orthogonal problem which arises even if there is some correct way to make tradeoffs.

We also have a \emph{proxy utility function} $\hat{u} : \mathbb{R}^k \to \mathbb{R}$ where $k < n$. The proxy is not sensitive to all features which true utility is sensitive to but only to a subset of them. This captures the idea that the proxy is a simplified version of the true utility function. This model can be motivated by limitations of measurement. It's impossible or difficult to measure all features relevant to value so we focus on a subset of features which we can easily measure. This fits with our initial motivation. For example, health is a complex value which depends on many factors, many of which are hard to measure. But some factors like step count are easy to measure, so we use them as proxy for health.

Being sensitive to fewer features is not the only way to model proxies. Call this a \emph{subset proxy}. There are other ways to think about proxies. The proxy could be sensitive to the same features as the true utility but have a simpler functional form. For example, suppose true utility is a weighted sum $\alpha x + \beta y$ where the weights $\alpha$ and $\beta$ measure how important features $x$ and $y$ are. We might have no idea what the weights are and use the unweighted sum $x + y$ as a proxy. This is different from a subset proxy. Alternatively, it might be that proxy utility and true utility are defined on disjoint sets of features. Perhaps we only care about feature $x$ but use feature $y$ as proxy because we believe $x$ and $y$ are correlated \parencite{Laidlaw2025}. This kind of \emph{correlated proxy} is also not directly covered by the model but as we will see later, we can relax some assumptions to cover it. These are limitations to the model. Nonetheless, subset proxies are worth studying and well-motivated by limitations of measurement.

The true utility function can be as complex as we like. If we care about fairness and justice, for example, we can build these concerns into the utility function, `consequentializing' as much as we like. This might be a worry if we are concerned with defending consequentialism as a substantive ethical theory \parencite{Brown2011} but that is not our purpose here. Rather, it's good that many different ethical frameworks can be represented by the utility function. We are \emph{not} assuming that we can explicitly write down the true utility function and use it for decision-making---presumably, it is too complicated for that.

Is the true utility function unique? A \emph{positive affine transformation} of $u$ is any $v = au + b$ with real numbers $a > 0$ and $b$. In decision theory, the utility function is normally only unique up to positive affine transformations. Some of the conditions discussed later require even less uniqueness. So the model does not commit us to a single true utility function but only to a set of utility functions up to some transformation. This is good news because it means that you can accept the model even if you are skeptical that value can be quantified by a unique utility function. But for now, it's best to pretend as if there is a unique utility function. We'll throw away the ladder later.

We have modeled states as points in $n$-dimensional Euclidean space. \textcite{Zhuang2020} assume that not every state is feasible. There is a cost function $c : \mathbb{R}^n \to \mathbb{R}$ which measures how costly a state is to realize and state $s$ is feasible only if $c(s) \leq 0$. This captures the idea that we have finite resources and can't maximize all features at the same time. Furthermore, features are constrained by their lower bound, so the set of feasible states is $\mathcal{S} =\{ s \in \mathbb{R}^n : c(s) \leq 0 \textrm{ and } s_i \geq b_i \textrm{ for every $1 \leq i \leq n$}\}$. \textcite{Zhuang2020} assume that $u$, $\hat{u}$ and $c$ are all strictly increasing in each feature and continuous. The assumption of strictly increasing utilities captures the idea that features are good. Other things equal, we want more of each. 

At a first pass, proxy failure happens if, as we maximize the proxy, true utility eventually goes down, like shown in Figure \ref{proxy-fail}. Before turning to the question under what conditions proxy failure happens, it is helpful to look at an example. 

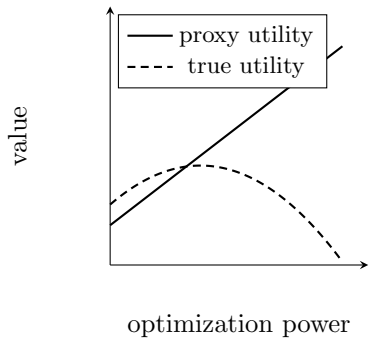
\begin{figure}
\centering
\begin{tikzpicture}
\centering
\begin{axis}[
    width=5cm,
    height=5cm,
    xlabel={optimization power},
    ylabel={value},
    xmin=0, xmax=10,
    ymin=-0.5, ymax=6,
    xtick=\empty,
    ytick=\empty,
    axis lines=left,
    legend style={
        at={(0.03,0.97)},
        anchor=north west,
        draw=black,
        fill=white,
        font=\small,
    },
    clip=false,
]

\addplot[
    black,
    solid,
    thick,
    domain=0:9,
    samples=100,
] {0.5*x + 0.5};
\addlegendentry{proxy utility}

\addplot[
    black,
    densely dashed,
    thick,
    domain=0:9,
    samples=100,
] {-0.08*(x-3.5)^2 + 2};
\addlegendentry{true utility}

\end{axis}
\end{tikzpicture}
\caption{Proxy failure.}\label{proxy-fail}
\end{figure}

Suppose we care about two features $x$ and $y$, for example happiness and beauty. States are points in two-dimensional Euclidean space. Assume true utility is the sum of $x$ and $y$. The set of feasible states is $\mathcal{S} = \{ \langle x,y \rangle \in \mathbb{R}^2:  2^x + 2^y \leq 20 \textrm{ and } x \geq 0 \textrm{ and } y \geq 0\}$, which means that at some point we have to make tradeoffs. Figure \ref{2dvalue} represents this example where lighter shades indicate higher true utility, visualizing gradients of bliss. The states below the arc going from approximately $\langle 0,4.25 \rangle$ to  $\langle 4.25,0 \rangle$ are feasible. The feasible state with maximum utility is found by following the diagonal line where $x=y$ to the boundary of the feasible region at $x = y \approx 3.32$. 

Consider the proxy $\hat{u}(\langle x,y \rangle) = x$. Maybe beauty can't be measured easily so we just optimize for happiness. If we maximize this proxy, we go as far right on the $x$-axis as possible. At some point we have to decrease the value on the $y$-axis to keep the cost below zero. We end up at the lower right corner with true utility around 4.25 which might be lower than at our starting point, for example if we start at $\langle 3,3 \rangle$. This shows how proxy failure as depicted in Figure \ref{proxy-fail} can arise.

\begin{figure}
\centering
\includegraphics[width=150pt]{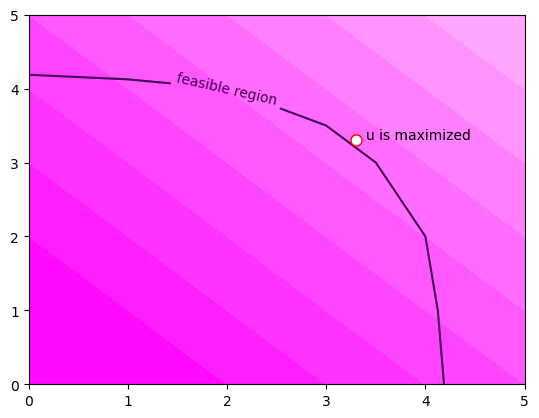}
\caption{Two-dimensional value.}\label{2dvalue}
\end{figure}

While this is a simple toy example, \textcite{Zhuang2020} observe similar behavior in a more complicated example. \textcite{Gao2023} observe similar overoptimization curves in a more realistic setting of reinforcement learning with a simplified reward model. This leads to the question: how general is proxy failure? Under what conditions do we see this kind of pattern?

\section{Proxy Failure}

\textcite{Zhuang2020} claim that proxy failure happens under plausible conditions. I explain their argument and raise some problems. Then, I suggest an improved story about when proxy failure happens.
 
Fix a true utility function $u : \mathbb{R}^n \to \mathbb{R}$. Consider any proxy $\hat{u} : \mathbb{R}^k \to \mathbb{R}$ with $k < n$. There are some features that the true utility function is sensitive to but the proxy is not. Call these \emph{unmentioned features}.  What happens to  unmentioned features as we maximize proxy utility?

Let's look at how \textcite{Zhuang2020} model optimization. An \emph{optimization sequence} is a sequence of feasible states $s^0, s^1, ...$ which converges to a feasible state where proxy utility is maximized, so $\lim_{t \to \infty} \hat{u}(s^t) = \sup_{s \in \mathcal{S}} \hat{u}(s)$.\footnote{The extreme value theorem says that a continuous function attains a maximum value on a compact set \parencite[p.~89]{Rudin1976}, so there is some feasible state where proxy utility is maximized. Note that while $\mathcal{S} =\{ s \in \mathbb{R}^n : c(s) \leq 0 \textrm{ and } s_i \geq b_i \textrm{ for every $1 \leq i \leq n$}\}$ is compact, $\{ s \in \mathbb{R}^n : c(s) \leq 0 \}$ is not compact. This will become relevant below.} We make no further assumptions about how the optimization sequence moves through state space. 

Say that a feature is minimized if it is equal to its lower bound. \textcite{Zhuang2020} show:
\begin{theorem}
If the optimization sequence $s^0, s^1,....$ converges to some $s \in \mathcal{S}$, then all unmentioned features are minimized at $s$.
\end{theorem}
As proxy utility is maximized, all unmentioned features are driven to their lowest possible value. If they are not, we can improve proxy utility by decreasing unmentioned features and increasing some feature the proxy is sensitive to. This looks bad since we're allocating all resources to the things we can measure and none to the things we can't measure. We have to make tradeoffs since the cost function is strictly increasing. If features are dimensions of value, the worry is that from the perspective of true utility, such `interdimensional tradeoffs' turn bad at some point but we'll keep making them anyways, which renders proxy optimization pathological.

Here is an example. Suppose our department measures and maximizes publication count as proxy for research success. At first, this increases our research success. But we have limited time so we start making shady tradeoffs. There are features not tracked by our proxy, for example how significant and original the published papers are. Since we just care about publication count, we stop caring about these other things and flood journals with variations of the same paper. Publication count keeps increasing but  research success is going down.

The theorem only requires that proxy utility is strictly increasing in each feature that it is defined on, not that true utility is strictly increasing in each feature. As I explain in the appendix, this allows us to capture correlated proxies. Also note that the theorem says that proxy optimization leads to all unmentioned features being minimized \emph{in the long run}, as the optimization sequence converges to an optimal state. The most we can show on this basis is that proxy optimization is bad in the long run. This leads to familiar worries about long-run convergence, since ``\emph{in the long run} we are all dead" \parencite[p.~80]{Keynes1923}.\footnote{Long-run convergence and its limits are often discussed in Bayesian epistemology in the context of convergence results \parencite{Nielsen2021}.} We can't say anything about rates of convergence without substantive assumptions about what the optimization sequence looks like. 

The theorem does not yet characterize proxy failure. Minimizing all unmentioned features has bad vibes but sometimes, proxy and true utility can be maximized at the same time. However, it suggests the following argument:\footnote{This quote by Stuart Russell explains the intuition behind the argument:
\begin{quote}
A system that is optimizing a function of $n$ variables, where the objective depends on a subset of size $k<n$, will often set the remaining unconstrained variables to extreme values; if one of those unconstrained variables is actually something we care about, the solution found may be highly undesirable. This is essentially the old story of the genie in the lamp, or the sorcerer's apprentice, or King Midas: you get exactly what you ask for, not what you want.
\end{quote}
The quote is from a discussion contribution found here: \url{https://www.edge.org/conversation/the-myth-of-ai\#26015}.}
\begin{enumerate}
\item Maximizing proxy utility leads to all unmentioned features being minimized.
\item It's bad if all unmentioned features are minimized.
\item Therefore, maximizing proxy utility is bad.
\end{enumerate}
The first premise is given by Theorem \ref{thm1a}. But how can we support the second premise? In the rest of this section, I discuss different ways to do so.

\subsection{Compactness}

\textcite{Zhuang2020} suggest the following assumption:
\begin{quote}
\textbf{Compactness}: for any $x \in \mathbb{R}$, the intersection of $\{ s \in \mathbb{R}^n : u(s) \geq x\}$ and $\{ s \in \mathbb{R}^n : c(s) \leq 0\}$ is compact.
\end{quote}

Intuitively, a subset of Euclidean space is compact if you can't keep going in any direction without eventually leaving the set. For example, a closed disk (the area bounded by a circle including the boundary line) is compact. The Heine-Borel theorem says that a subset of Euclidean space is compact if and only if it is closed and bounded \parencite[p.~40]{Rudin1976}. So one way not to be compact is to be unbounded, like the region above the diagonal line where $x=y$. Another way not to be compact is to be bounded but topologically open, for example an open disk (the interior of a circle with the boundary line removed). 

Compactness says that for any $x \in \mathbb{R}$, the set of states which have utility at least $x$ and are feasible according to the cost function is compact. This means that you can't keep going in any direction without eventually leaving the set: either utility decreases below $x$ or you hit the feasibility constraint. Note that the definition of Compactness cares only about the cost function and ignores the lower bounds. A state might be feasible according to the cost function in the sense that $c(s) \leq 0$ but ruled out by the lower bounds, a detail which matters for the theorem below.

The example shown in Figure \ref{2dvalue} satisfies Compactness and also exhibits proxy failure, so there is hope that Compactness gives us a general story about when proxy failure happens. Figure \ref{not-compact} shows an example where Compactness fails. The feasible states according to the cost function are $\{ \langle x,y \rangle \in \mathbb{R}^2 : x + y \leq 5\}$ (ignoring lower bounds), true utility is $x + 2y$ and proxy utility is $y$. There is no proxy failure. No matter where we draw the lower bounds, we can maximize the proxy and true utility at the same time.

To see why Compactness fails, consider the set of states $s$ with $u(s) \geq 1$. This is the set of states above the diagonal line passing through $(0,0.5)$ and $(1,0)$ in the lower left corner of Figure \ref{not-compact}. We intersect this set with the set of feasible states according to $c$. The intersection of these sets is not compact because we can keep going to the upper left without bound. In particular, we can follow the dotted line where $x+2y = 4$.

\begin{figure}
\centering
\includegraphics[width=150pt]{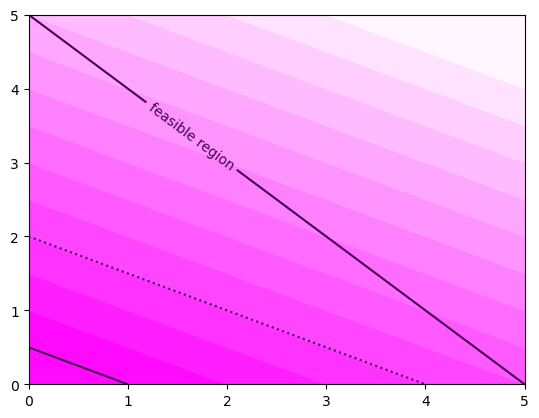}
\caption{No compactness.}\label{not-compact}
\end{figure}

Compactness is preserved by positive affine transformations:\footnote{Both here and below, assume we are holding the rest of the framework ($c$, $\mathcal{S}$, $\hat{u}$, $s^0$) fixed.}
\begin{theorem}
If $u$ satisfies Compactness, then any positive affine transformation of $u$ satisfies Compactness.
\end{theorem}
This tells us how much realism about utility is required by Compactness. Since Compactness is preserved under positive affine transformations, we don't have to think there is a unique real-valued utility function but only a family of utility functions up to positive affine transformations, which is a standard uniqueness condition \parencite{Ramsey1926, vonNeumann1944, Savage1972}. However, we will later see conditions requiring less uniqueness and correspondingly less realism about utility.

To formalize the idea of optimization turning bad from the point of view of true utility, \textcite{Zhuang2020} say that for any $x \in \mathbb{R}$, optimization is \emph{$x$-costly} if, for our optimization sequence $s^0, s^1, ...$, we have $\limsup\limits_{t \to \infty} u(s^t) \leq x$. This means that, as we maximize proxy utility, true utility eventually falls and stays below $x$. They show:
\begin{theorem}
Compactness holds if and only if, for any proxy $\hat{u}$ and any $x \in \mathbb{R}$, there is some unmentioned feature $k$ and some $r \in \mathbb{R}$ such that, if the lower bound of $k$ is below $r$, then optimization is $x$-costly.
\end{theorem}
In this theorem, we keep $c$ and $u$ fixed and quantify over different choices of lower bounds. The intended interpretation is that proxy optimization eventually becomes arbitrarily costly in terms of true utility. They write that under the assumptions of the theorem, ``we can guarantee that as optimization progresses, eventually overoptimization occurs'' \parencite[p.~5]{Zhuang2020}. This seems disturbing: proxy optimization causes massive harm as all resources are directed away from unmentioned features.

However, there are two reason to worry about this story. First, the theorem does not say that proxy optimization actually leads to true utility decreasing. The theorem only says that proxy optimization leads to true utility decreasing if the lower bound of some unmentioned feature is sufficiently low. To see this, take another look at Figure \ref{2dvalue} introduced before. Compactness holds. Suppose we start at $\langle 2,2 \rangle$ with true utility 4 and maximize proxy utility $\hat{u}(\langle x,y \rangle) = x$. As we saw earlier, we end up at the lower right corner with true utility around 4.25, so we are better off than at our starting point. Proxy optimization does not lead true utility to decrease. The reason we don't keep going is that the $y$-feature can't be decreased any further. The lower bound rescues us. So the argument against proxy optimization based on the theorem is limited. The theorem only shows that proxy optimization would be bad if lower bounds were low but they might not be, so it does not explain proxy failure as depicted in Figure \ref{proxy-fail}. In one way, Compactness is too weak.

Second, Compactness rules out reasonable utility functions. In particular, Compactness rules out utility functions which are bounded below.\footnote{Proof: Assume $u$ is bounded below. Then for some $x \in \mathbb{R}$, $u(s) \geq x$ for all $s \in \mathbb{R}^n$. So $\{s \in \mathbb{R}^n: u(s) \geq x\} \cap \{s \in \mathbb{R}^n: c(s) \leq 0\} = \{s \in \mathbb{R}^n: c(s) \leq 0\}$ which is not compact since $c$ is strictly increasing.} This is a strong and implausible requirement. Surely, proxy failure can happen even if there is a limit to how bad things can get. Suppose we stick with our example but change true utility to be bounded below. A simple way to do so is the sigmoid function $\sigma(x+y) = \frac{1}{1 + e^{-(x+y)}}$. This utility function is bounded above and below so Compactness fails. However, we get a similar kind of proxy failure as before. Maximizing true utility brings us to approximately $\langle 3.32, 3.32 \rangle$ while proxy optimization pushes us to one of the corners. So Compactness does not get at the heart of proxy failure. Further, there are  reasons to reject unbounded utility functions because they lead to problems in decision making under uncertainty \parencite{Russell2021}. So in another way, Compactness is too strong.

\subsection{Balance}

Here is a better story, based on the following idea: in feasible states where true utility is maximized, not all unmentioned features are minimized. We can write this as follows:
\begin{quote}
\textbf{Minimal Balance}: for all $s \in \mathcal{S}$ which maximize true utility, there is some unmentioned feature not minimized in $s$.
\end{quote}
If Minimal Balance holds, then by Theorem \ref{thm1a}, any state which maximizes proxy utility does not maximize true utility. So, given Minimal Balance, minimizing all unmentioned features is bad in some sense: it does not lead to maximum true utility. This supports the second premise of our argument against proxy optimization. Like Compactness, Minimal Balance constrains the combination of true utility and feasible states. It also constrains the proxy since which features are unmentioned depends on the proxy. However, Minimal Balance is logically independent of Compactness.\footnote{To see that Compactness doesn't entail Minimal Balance, let $\mathcal{S} = \{ \langle x,y \rangle \in \mathbb{R}^2 : 2^{(x+3)} + 2^y \leq 20 \textrm{ and } x \geq 0 \textrm{ and } y \geq 0\}$ with $u(\langle x,y \rangle) = x + 2y$ and $x$ unmentioned. Compactness is satisfied but Minimal Balance fails since the feasible state maximizing $u$ minimizes $x$. For the other direction, consider the example of sigmoid utilities which violates Compactness but satisfies Minimal Balance since $\arg \max_{s \in \mathcal{S}} \sigma(x + y) \approx \langle 3.32, 3.32 \rangle$.}

Minimal Balance expresses the idea that some unmentioned features matter. Maximizing overall value requires at least some allocation of resources to some unmentioned features. This seems plausible given our limitations of measurement: we might not be able to measure all features which matter for overall value. Roughly, we can think of Minimal Balance as encoding some ideal of a `well-rounded world'.

Minimal Balance allows some tradeoffs. We can decrease unmentioned features and thereby increase true utility. But according to Minimal Balance, unmentioned features matter enough so that minimizing all of them can't be offset by increasing the features tracked by our proxy. So there are constraints on what tradeoffs we can make. For example, if we care about maximizing research success, perhaps it's okay to decrease significance per paper a little bit if we publish many more papers. But if we minimize significance, no amount of increase in publication volume can make up for that. The key idea is that if we keep making tradeoffs which decrease unmentioned features, these tradeoffs will eventually push us away from the global optimum.

We can give a heuristic justification for Minimal Balance in terms of bargaining. Imagine features as agents in a bargaining problem debating how to distribute fixed resources. Not all features can be maximized so we must make tradeoffs. The true utility function represents the optimal solution to the bargaining problem and the proxy represents the optimal solution for a proper subset of agents. Unmentioned features correspond to agents whose concerns are not represented by the proxy. In this context, Minimal Balance says that the optimal solution to our bargaining problem does not fail to give any resources to agents not represented. This is plausible and fits with axiomatic models of bargaining \parencite{Nash1950}.

Minimal Balance requires much less uniqueness of the true utility function than Compactness since it is invariant under any transformation which preserves global maxima: 
\begin{theorem}
If $u$ satisfies Minimal Balance, then any max-preserving transformation of $u$ satisfies Minimal Balance.
\end{theorem}
The proof is in the appendix. This means that you can accept Minimal Balance without thinking that there is a true utility function unique up to positive affine transformations. You don't even have to think that true utility gives us an order of states since max-preserving transformations need not preserve the order. All that matters is that there is some utopian state where things are best. This seems like a modest requirement and so supports the plausibility of Minimal Balance.

However, achieving a utopian global optimum is a lot to ask. Perhaps what matters is increasing true utility relative to the status quo. Even if Minimal Balance holds, proxy optimization might increase true utility relative to our starting state, as in the example shown by Figure \ref{2dvalue} starting at $\langle 2,2 \rangle$. Minimal Balance holds but proxy optimization beats doing nothing.

Here is a plausible strengthening of Minimal Balance: in the status quo, features are somewhat balanced. Some unmentioned features which significantly contribute to value are not minimized. Further, if these features were minimized, things would be worse than they are. We can write this as:
\begin{quote}
\textbf{Actual Balance}: some unmentioned feature $k$ is not minimized in $s^0$ and for any $s \in \mathcal{S}$ where $k$ is minimized, $u(s^0) > u(s)$.
\end{quote}
Actual Balance entails Minimal Balance.\footnote{Proof: assume $s^* \in \mathcal{S}$ maximizes true utility, so $u(s^*) \geq u(s^0)$. By Actual Balance, $u(s^0) > u(s)$ for any $s$ where all unmentioned features are minimized. So $u(s^*) > u(s)$, which means that in $s^*$, some unmentioned feature is not minimized.} By Theorem \ref{thm1a}, Actual Balance implies that proxy optimization will eventually lead true utility to decrease relative to our starting state. If we keep making tradeoffs which decrease unmentioned features, this eventually pushes true utility below our starting point.

Like Minimal Balance, Actual Balance allows some tradeoffs. It does not say that decreasing unmentioned features always makes things worse. It just says that minimizing some unmentioned features makes things worse. This seems plausible if some unmentioned features significantly contribute to value.

For example, suppose that in the status quo, the significance of papers published by our department is not minimal. Perhaps we can increase research success while decreasing  significance per paper a little bit, but minimizing significance is  worse in terms of research success than the status quo, no matter how many papers we publish.

The idea is that the world is at least somewhat well-rounded. There might be many ways to improve things, but minimizing all features not included in our proxy will make things worse. This captures the idea that while the world might be far from optimal, there is some reason for the way things are and we should not radically change things if we can't measure everything that matters, which is sometimes illustrated with the story of Chesterton's Fence.

So Actual Balance provides a robust justification of the second premise of our argument against proxy optimization. It identifies a plausible constraint under which proxy optimization leads to worse outcomes than doing nothing. This fits with the picture of proxy failure in Figure \ref{proxy-fail}. One difference is that Actual Balance is compatible with true utility going down immediately once we start proxy optimization. However, this is a helpful refinement to our initial conception of proxy failure. Sometimes, optimization makes things worse immediately. Figure \ref{proxy-fail} depicts a situation in which the proxy is `locally valid' in the sense that in some neighborhood around the starting state, increasing the proxy also increases true utility. If we're smart about constructing a proxy, we can hope that our proxy is at least locally valid. But if we keep optimizing, the proxy might no longer be valid.

Actual Balance is invariant under (strict) monotonic transformations which preserve the ordering of states:
\begin{theorem}
If $u$ satisfies Actual Balance, then any strict monotonic transformation of $u$ satisfies Actual Balance.
\end{theorem}
So Actual Balance requires a uniqueness commitment between Minimal Balance and Compactness. You don't have to think that true utility is unique up to positive affine transformations but only that there is some order of states in terms of how good they are. (The order need not be complete.) This is much weaker than the uniqueness requirement of Compactness and so supports the plausibility of Actual Balance. Compactness is not invariant under monotonic transformations since we can monotonically transform an unbounded utility function which satisfies Compactness to a bounded utility function which does not. An example illustrating the basic idea is shown in Figure \ref{sigmoid}. We have feature $x$ with unbounded utility function $u(x) = x$. The sigmoid function $\sigma(x) = \frac{1}{1 + e^{-x}}$ is a monotonic transformation of $u$ but bounded above and below.

\begin{figure}
\centering
\includegraphics[width=150pt]{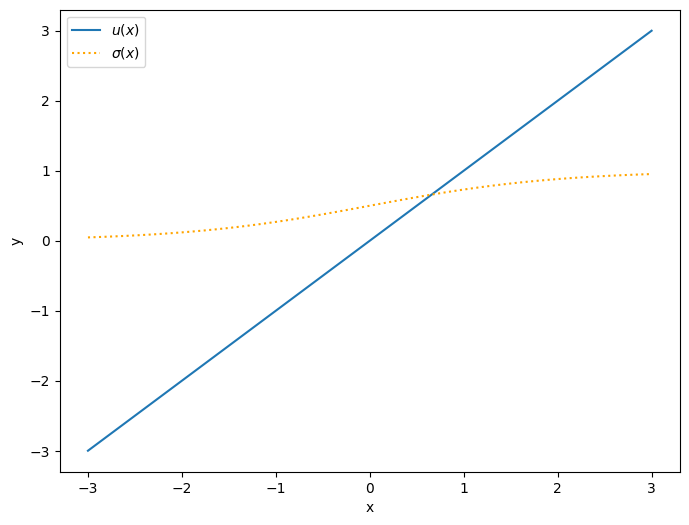}
\caption{Monotonic transformation.}\label{sigmoid}
\end{figure}

Take a step back. We discussed how to justify the idea that minimizing all unmentioned features is bad. We found two plausible principles which support this idea. Minimal Balance says that minimizing all unmentioned features is bad in the sense of not maximizing true utility. Actual Balance says that minimizing all unmentioned features is bad in the sense of decreasing true utility relative to our starting point. Both principles are motivated by the idea that unmentioned features contribute significantly to value. So we have the outlines of a plausible story about why proxy optimization is bad.

The model of proxy optimization studied here could be refined in many ways. As it stands, the model is itself only a proxy for the phenomenon we are interested in and many important factors are not part of the model, such as different kinds of proxies, uncertainty and strategic interaction. Most importantly, the results are about what happens in the long run. It would be desirable to characterize proxy failure in the short run. As mentioned, this requires more restrictive assumptions on the optimization process. Nonetheless, it seems foolish to hope that we're all dead before we see proxy failure. So long-run results are an important first step to investigate proxy failure even if they don't answer all the questions we care about.

\subsection{Proxy Loss}

A key motivation for proxies is that features relevant to true utility might be costly to measure. This is, for example, why we use GDP as proxy for welfare. It might also be costly for other reasons to find and maximize the true utility function, for example because it's hard to find out what really matters. But if proxy optimization is bad, it might be better to pay these costs and maximize true utility. This depends on how bad proxy optimization is. 

If maximizing true utility has cost $c > 0$, the overall value of maximizing true utility while paying the cost is $max_{s \in S} u(s) - c$. We can compare the overall value of maximizing true utility while paying the cost to the value of maximizing the proxy $\hat{u}$. Assume for now that there is a unique $s^* \in \mathcal{S}$ which maximizes the proxy. Also assume that true utility and cost are unique up to positive affine transformations.

We can introduce notions of proxy loss which measure how much utility we lose by proxy optimization. Define \emph{ideal proxy loss} $\ell_1$ as $max_{s \in S} u(s) - u(s^*)$. This measures how much utility we lose by maximizing the proxy instead of the true utility function. Minimal Balance ensures that ideal proxy loss is positive. Ideal proxy loss is analogous to the cost of measurement above, so maximizing true utility while paying the cost is better than proxy optimization iff $c \leq \ell_1$. 

Define \emph{actual proxy loss} $\ell_2$ as $u(s^0) - u(s^*)$. This measures how much utility we lose by maximizing the proxy instead of sticking with the status quo. Actual Balance ensures that actual proxy loss is positive. It follows from the definitions that $\ell_1 \geq \ell_2$. If $\ell_1$ is positive but $\ell_2$ is negative, this means that while we lose utility by maximizing the proxy instead of true utility, we still improve our situation relative to the status quo.

In general, there is no guarantee that there is a unique maximizer of proxy utility. Without uniquness, things get more complicated. We can do a worst-case analysis by looking at the maximizer of proxy utility with the lowest true utility. We can also introduce probabilities over optimization sequences. I leave these refinements for another time.

\section{A Dilemma for Optimization}

We have the mathematical framework of decision theory and optimization but how do we apply this framework to a real-world problem? Decision theory advises us to maximize some utility function but what utility function do we plug in? We face a dilemma. Either we use the true utility function or we use a proxy. In many cases, we can't use the true utility function because it is completely intractable. So we have to use some proxy. But as we've seen, under plausible assumptions, maximizing proxy utility is eventually bad from the perspective of true utility.

One way to frame this dilemma is in terms of broad and narrow utility functions. Broad utility functions are sensitive to many features of the world. They are fine-grained and complicated and for this very reason can plausibly be thought to capture the true complexity of our values. While broad utility functions are plausible as measures of value, we typically can't explicitly write them down and use them to make decisions. In contrast, narrow utility functions are sensitive to only a few features of the world, like money. We can explicitly write them down but they are less plausible as measures of value, at least for most of us. The dilemma is that when applying decision theory, both broad and narrow utility functions are problematic. Broad utility functions are problematic because we can't write them down and use them. Narrow utility functions are problematic because they are subject to proxy failure. So how can we ever apply decision theory?\footnote{We can get another angle on this problem by recalling an old debate on induction. Carnap and Hempel developed formal systems of inductive reasoning. \textcite{Goodman1946} objects that there is a problem with applying this `ingenious and valuable logico-mathematical apparatus' because we have to decide which predicates are primitive. Some choices lead to a reasonable model of inductive reasoning while others, like grue, do not. \textcite{Carnap1947} agrees and suggests we need another theory which tell us how to apply inductive logic. Analogously, we have the ingenious mathematical apparatus of decision theory and optimization that we can apply to the correct value function whenever we discover what the correct value function is. Sadly, it seems like neither broad nor narrow utility functions will do the job.}

Note that, in our framework, proxy failure happens in the long run. So the argument is compatible with proxy optimization improving utility in the short run. This is a good thing because we sometimes use proxy optimization to make things better. What the argument shows is that, under plausible assumptions, proxy optimization will eventually turn bad. So when doing optimization, we should be careful to notice when proxies are failing. In general, knowing when proxies are failing is a difficult problem. 

The situation seems even worse when it comes to AI systems which are trained to maximize some proxy. There are two reasons to think that proxy failure is of special relevance to AI. First, AI systems are often trained using large amounts of optimization pressure, making proxy failure more likely. The amount of optimization pressure will likely increase if current trends of compute scaling continue. Second, the scale of deployment of AI systems makes proxy failure potentially more harmful. If an AI system is deployed at a planetary scale, proxy failure matters much more than for a single human agent.

\section{Mitigation strategies}

I discuss some responses to proxy failure. All of them face serious problems, so the upshot is that the problem of proxy failure is quite robust.

\subsection{Obey the axioms}

Perhaps it's a mistake to think that we apply decision theory by specifying a utility function and then maximizing this  function. Instead, we start with preferences. Decision theory tells us that our preferences should satisfy certain axioms, such as transitivity: if you prefer $a$ to $b$ and $b$ to $c$, you should also prefer $a$ to $c$. Then we invoke representation theorems which say that preferences satisfy the axioms if and only if they are representable as maximizing (expected) utility with respect to some utility function.\footnote{Some representation theorems presuppose probability \parencite{vonNeumann1944}, others derive probability from preferences \parencite{Ramsey1926, Savage1972, Jeffrey1983}. There are also representation theorems for utility without uncertainty \parencite{Kreps1988, Kobberling2006}. \textcite{Krantz1971} discuss many relevant results in measurement theory.} The key point is that if you satisfy the axioms, you get the utility function `for free'. This way of thinking about decision theory looks like it allows us to sidestep constructing a proxy. Just obey the axioms and you're guaranteed to maximize your true utility function even if we have no way to write down this utility function.

However, there is a problem. This mitigation strategy provides extremely thin normative guidance. Obeying the axioms ensures that you maximize \emph{some} utility function with no restrictions on what this utility function looks like. This means that you are coherent in some way, but we want more than coherence, especially when setting proxies for other agents. To make this vivid, imagine we know that an AI system obeys the right axioms and so maximizes (expected) utility with respect to some utility function. If we're concerned about whether this AI system is safe, should this reassure us? Not at all. The AI system could still kill everyone. More broadly, it's not clear what is so great about coherence if your aims are not substantively good \parencite{Kolodny2005}. 

You might respond that the AI system should obey the axioms and learn our preferences. But this seems no easier than constructing a proxy. The lesson of the representation theorems is that learning the preferences of an agent and learning their utility function are equivalent under appropriate assumptions. This brings out how the mitigation strategy does not really avoid the problem of proxy failure. If you maximize some proxy, you maximize some utility function and so you obey the axioms. However, the proxy could diverge arbitrarily from what we care about. Obeying the axioms is not enough.

\subsection{Change the proxy}

We've seen that under plausible conditions, maximizing any fixed proxy is bad in the long run. Perhaps we should change the proxy over time. The rough idea is that you should change the proxy if you start overoptimizing, then maximize the new proxy for a while until you start overoptimizing again, and so on. By chaining together locally valid proxies in this manner, we can seemingly avoid the perils of proxy failure. A similar idea is discussed by \textcite{Nguyen2023} in an informal setting: we have to keep updating our heuristics to avoid manipulation in a hostile epistemic environment. \textcite{Wilson2006} discusses related issues about approximation strategies in applied mathematics.

\textcite{Zhuang2020} make this idea precise. Define the \emph{sensitivity} of feature $i$ in state $s$ as measuring how much true utility changes as a result of a (normalized) change in the feature.\footnote{They define the sensitivity of feature $i$ in $s$ in terms of partial derivatives: $\frac{\partial u}{\partial s_i}(\frac{\partial c}{\partial s_i})^{-1}$.} Then construct proxy $\hat{u}$ defined on the features with the largest and smallest sensitivity. They show that if the proxy always tracks the features with largest and smallest sensitivity and we keep other features constant, we will converge to a feasible state where true utility is maximized, assuming Compactness. The intuitive idea is that we only make tradeoffs between features which are not important (small sensitivity) and features which are important (large sensitivity) so we avoid the pathological behavior of proxy optimization.

The problem with this mitigation strategy is that it relies on strong assumptions. First, since in different states different features need to be tracked by our proxy, it assumes that in principle we can measure each feature. But some features relevant to true utility might be difficult or impossible to measure. Indeed, limitations of measurement are the key reason why we are talking about proxies in the first place. If we can measure all features, it's not clear why we should use proxies at all. 

Further, the strategy assumes that we can measure the sensitivity of each feature in each state, which requires us to measure true utility. Again, the reason we are talking about proxies in the first place is that true utility is often hard or impossible to measure. The result also requires that all unmentioned features are kept constant during proxy optimization. Since we might not be able measure these features, this seems difficult or impossible to do. Finally, the result relies on Compactness. As mentioned above, Compactness has the problematic consequence of ruling out bounded utility functions.

So the second mitigation strategy also faces serious problems. However, the broad lesson that we can sometimes avoid proxy failure by changing our proxy over time seems correct. We can also use a set of proxies and make sure we are doing well on all of them, which is a technique used in machine learning \parencite{Coste2023}. Still, even a set of proxies can miss out on features which matter if we can't measure the relevant features, so this strategy does not completely solve the problem of proxy failure. 

\subsection{Chill optimization}

Here is a variant of the second mitigation strategy. We construct the best proxy we can come up with. Then we start proxy optimization. We can't directly measure true utility (otherwise why use proxies in the first place?) and so we can't measure the sensitivity of each feature. However, assume that we know whether true utility is going up or down. We have an overall sense of things getting better or worse even if we can't say exactly how much. If true utility starts going down, stop proxy optimization. Then, see if we can construct a better proxy and repeat.

This strategy is not guaranteed to converge to a global optimum. It's not even guaranteed to increase true utility relative to our starting point since true utility might go down immediately. However, this strategy avoids proxy failure since we stop optimizing if things are getting worse. 

So we get weaker performance than with the kind of proxy change discussed by \textcite{Zhuang2020}. However, we also make weaker assumptions: we might not be able measure all features or true utility. All we know is whether true utility is going up or down. We might not be able to construct the optimal proxy, perhaps because the features with largest and smallest sensitivity can't be measured. Instead, we simply try to do the best we can given our limitations of measurement. \textcite{Karwowski2024} suggest a similar strategy of `early stopping' to avoid overoptimization in reinforcement learning.

This strategy seems like a plausible remedy for proxy failure in the context of public policy: if we notice things are getting worse, just stop optimizing. There might still be problems since we might have legal or institutional reasons to continue optimizing even if proxies are failing, but these problems are outside the scope of our framework. 

It also seems plausible for AI alignment. In this context, the present strategy says to let an AI system maximize some proxy until we notice things are getting worse, then switch it off. This assumes that AI systems will let themselves be switched off, which might fail \parencite{Hadfield2017, Neth2025}. There is some evidence for shutdown resistance in current LLMs, although it's unclear how much of this should be understood as role-play \parencite{Lynch2025, Mitchell2025}. Even so, the problem of ensuring AI systems will let themselves be switched off is arguably more tractable than the problem of constructing an optimal proxy for each state. 

However, the strategy assumes that we know whether true utility is going up or down. This is hard to know in general. If utility is complex and depends on factors that we can't measure, it might be difficult to say whether true utility is going up or down. Also, the effects of proxy optimization on true utility might be delayed. So knowing when proxies are failing is a difficult problem, not just for AI systems but also for humans.

Even if we can tell whether true utility is going down, there is the further question of why it is going down. In the model, we assume that all changes are due to the optimization process. But maybe true utility is going down for other reasons. Perhaps proxy optimization is making a positive contribution but adversarial agents are making things worse overall. Or perhaps  completely unrelated events are reducing true utility, like natural disasters. Such factors are currently not part of the model. They also present a difficult causal attribution problem: how do we tell what causes changes in true utility? I leave the problem of including such factors for another occasion. 

If changes are due to other causes, this might prevent proxy optimization from minimizing all unmentioned features. Lack of control might prevent proxy failure. However, it seems unwise to trust that the world will always be this kind. Changes unrelated to optimization might also push us outside of regimes where proxies are locally valid. It would require a kind of pre-established harmony for such changes to always work in our favor. Overall, this makes the problem of spotting proxy failure even more difficult.

\section{Consequences}

What follows? First, we can give a decision-theoretic explanation of why what \textcite{Nguyen2024} calls value capture is often bad: because of proxy failure. We have identified plausible conditions under which proxy optimization eventually leads to proxy failure and decreases true utility. We can think of this as a decision-theoretic vindication of the idea that there is something bad about value capture. The model also provides additional insights, for example by making the notion of a proxy precise and highlighting the role of optimization. Using proxies might be okay if you don't optimize too hard. 

The model leads to a dilemma for optimization. True utility often can't be measured because we can't measure all features relevant to value. But maximizing proxies often leads to proxy failure. This poses problems for any attempt to use maximization of some function as a decision procedure. If we think of consequentialist decision procedures as doing so, we have an argument that consequentialist decision procedures are often self-defeating in the long run. This applies to attempts to use consequentialist decision-procedures in the real world, such as effective altruism and cost-benefit analysis. The metrics used by proponents of effective altruism and cost-benefit analysis are best understood as proxies and so subject to proxy failure. These problems might be mitigated by changing the proxy or stopping optimization once things are getting worse, but this requires departures from maximizing any fixed proxy.

Proxy failure provides a useful lens for reframing problems of AI safety. This applies both to current AI systems \parencite{Creel2022,Thomas2022} and potentially more powerful future AI systems \parencite{Dung2023}. In particular, proxy failure is a useful lens for understanding failure modes of reinforcement learning from human feedback (RLHF) applied to LLMs: models become sycophantic and deceptive because that is what users like \parencite{Pan2024, Williams2025}. Further, LLMs reward hack in coding and problem solving and obfuscate doing so \parencite{Chen2025}.

Two more consequences. The framework suggests a critique of policies focused exclusively on increasing GDP, which is best understood as proxy for what matters and so subject to proxy failure. In some cases, a plurality of indicators might help \parencite[pp.~23-29]{Piketty2022}. However, as discussed above, a set of proxies does not fully solve the problem of proxy failure.

Finally, there are upshots for debates about AI capabilities. Here is a familiar situation. AI systems manage to beat some benchmark for intelligence such as playing chess or creating memes. Skeptics respond by saying that this is not what we really mean by `intelligence'. The skeptics are accused of moving the goalpost. The framework sketched here suggests that moving the goalpost is often reasonable if understood as updating an imperfect proxy. Benchmarks are best understood as proxies for intelligence and when we realize that AI systems can pass the benchmark without general intelligence, we realize the ways in which this proxy is imperfect. This is arguably the right way to understand the Turing test. It's not a definition of intelligence but a proxy and once we realize ways in which it's not a good proxy we should update \parencite{Mitchell2024}. Moving the goalpost seems like right approach when evaluating complex and multidimensional concepts.

\section*{Appendix}

We can apply the model to choice under uncertainty. Think of each $s \in \mathcal{S}$ as an uncertain prospect where $s_i$ is the expected value of feature $i$ under some probability function. We can replace expected values by risk-weighted expected values or similar. All the results go through under this interpretation. One limitation is that under this interpretation, true utility is a function of the vector of expected values of all features. You might also care about correlations among features. However, this still shows that the model can be applied to choice under uncertainty.

\textcite{Zhuang2020} show
\begin{theorema}\label{thm1a}
If the optimization sequence $s^0, s^1,....$ converges to some $s \in \mathcal{S}$, then all unmentioned features are minimized at $s$.
\end{theorema}
\noindent I reproduce the proof for the sake of completeness. 
\begin{proof}
Assume the optimization sequence $s^0, s^1,....$ converges to some $s \in \mathcal{S}$.  Suppose for \emph{reductio} that there is some unmentioned feature $k$ with lower bound $b_k$ and $s_k > b_k$. We show that we can increase proxy utility by decreasing the unmentioned feature $k$ and increasing some mentioned feature instead.

Since $s \in \mathcal{S}$, we have $c(s) \leq 0$. The cost function $c: \mathbb{R}^n \to \mathbb{R}$ is strictly increasing in each feature so for any $x, y \in \mathbb{R}^n$ with $x_i > y_i$ for some feature $i$ and $x_j = y_j$ for all other features $j$, we have $c(x) > c(y)$. Since $c$ is strictly increasing in each feature, $c(s - \epsilon \mathbf{e}_k) < 0$ for every $\epsilon > 0$, where $\mathbf{e}_k$ is the standard basis vector in dimension $k$. This is the vector that has a one as $k$-th entry and zeros everywhere else and `$-$' denotes vector subtraction.

Since $s_k > b_k$, we can choose some $\epsilon > 0$ such that $s_k - \epsilon > b_k$, so $s - \epsilon \mathbf{e}_k$ is feasible both according to $c$ and according to the lower bound. Since feature $k$ is unmentioned, $\hat{u}(s - \epsilon \mathbf{e}_k) = \hat{u}(s)$.

Since $c$ is continuous, we can choose some $\delta > 0$ and some mentioned feature $j$ such that $c(s - \epsilon \mathbf{e}_k + \delta \mathbf{e}_j ) \leq 0$, so $s^\prime= s - \epsilon \mathbf{e}_k + \delta \mathbf{e}_j$ is feasible, where `$+$' denotes vector addition. Proxy utility is strictly increasing in each mentioned feature, so $\hat{u}(s^\prime) > \hat{u}(s)$. We have shown that $s^\prime$ is feasible and has higher proxy utility than $s$, so $s$ does not maximize proxy utility among the feasible states, so the optimization sequence does not converge to $s$.
\end{proof}

This proof requires proxy utility to be strictly increasing in each feature it is defined on but not true utility. This means we can capture correlated proxies. For example, suppose there are two features $x$ and $y$ with true utility equal to $x$ and proxy utility equal to $y$. Proxy utility is strictly increasing in the feature $y$ it is defined on. The theorem applies even though true utility is not strictly increasing in $y$. To capture the belief that $x$ and $y$ are correlated, we need to add a probability measure which I won't explore in detail. However, correlations are fragile. Features might be correlated under your prior but correlations might vanish if we apply optimization pressure, a point noted in discussions of Goodhart's law \parencite{Manheim2018}. Minimization of unmentioned features provides an explanation of why correlations disappear under optimization pressure. A probability measure would also allow us to model optimization as a stochastic process, allowing us to impose constraints on the stochastic drift of proxy optimization.

\begin{theorema}
If $u$ satisfies Compactness, then any positive affine transformation of $u$ satisfies Compactness.
\end{theorema}

\begin{proof}
Assume $u$ satisfies Compactness. Let $v = au + b$ with real numbers $a > 0$ and $b$. Consider an arbitrary $x \in \mathbb{R}$. We have $\{ s \in \mathbb{R}^n : v(s) \geq x\} = \{ s \in \mathbb{R}^n : au(s) + b \geq x\} =  \{ s \in \mathbb{R}^n : u(s) \geq \frac{x - b}{a}\}$ since $a > 0$. Since $ \frac{x - b}{a} \in \mathbb{R}$, by assumption $\{ s \in \mathbb{R}^n : u(s) \geq \frac{x - b}{a} \} \cap \{ s \in \mathbb{R}^n : c(s) \leq 0\}$ is compact.
\end{proof}

\stepcounter{theorema} 

\begin{theorema}
If $u$ satisfies Minimal Balance, then any max-preserving transformation of $u$ satisfies Minimal Balance.
\end{theorema}

\begin{proof}
Assume $u$ satisfies Minimal Balance. Let $v$ be a max-preserving transformation of $u$ so $\arg \max_{s \in \mathcal{S}} v(s) =  \arg \max_{s \in \mathcal{S}} u(s)$. Assume $s \in \arg \max_{s \in \mathcal{S}} v(s)$. Since $v$ is a max-preserving transformation of $u$, $s \in \arg \max_{s \in \mathcal{S}} u(s)$. By assumption, some unmentioned feature is not minimized in $s$, so $v$ satisfies Minimal Balance.
\end{proof}

\begin{theorema}
If $u$ satisfies Actual Balance, then any strict monotonic transformation of $u$ satisfies Actual Balance.
\end{theorema}

\begin{proof}
Assume $u$ satisfies Actual Balance. Let $v$ be a strict monotonic transformation of $u$, so for all $x, y \in \mathbb{R}^n$, $u(x) > u(y)$ implies $v(x) > v(y)$. By assumption, some unmentioned feature is not minimized in $s^0$ and any feasible state where that feature is minimized is worse than $s^0$ according to $u$. Consider any $s \in \mathcal{S}$ where that feature is minimized. By assumption, $u(s^0) > u(s)$, so since $v$ is a strict monotonic transformation of $u$, $v(s^0) > v(s)$, so $v$ satisfies Actual Balance.
\end{proof}

\singlespacing

\printbibliography

\end{document}